\documentclass[lettersize,journal]{IEEEtran}
\usepackage{amsmath,amsfonts}
\usepackage{array}
\usepackage[caption=false,font=normalsize,labelfont=sf,textfont=sf]{subfig}
\usepackage{textcomp}
\usepackage{stfloats}
\usepackage{url}
\usepackage{verbatim}
\usepackage{graphicx}
\def\BibTeX{{\rm B\kern-.05em{\sc i\kern-.025em b}\kern-.08em
    T\kern-.1667em\lower.7ex\hbox{E}\kern-.125emX}}
\usepackage{balance}
\usepackage{amsmath}
\usepackage{transparent}

\usepackage[utf8]{inputenc} % allow utf-8 input
\usepackage[T1]{fontenc}    % use 8-bit T1 fonts
\usepackage{hyperref}       % hyperlinks
\usepackage{url}            % simple URL typesetting
\usepackage{booktabs}       % professional-quality tables
\usepackage{amsfonts}       % blackboard math symbols
\usepackage{nicefrac}       % compact symbols for 1/2, etc.
\usepackage{microtype}      % microtypography
\usepackage{xcolor}         % colors
\usepackage{stfloats}
\usepackage[utf8]{inputenc}
\usepackage[T1]{fontenc}
\usepackage{booktabs}
\usepackage{amsfonts}       % blackboard math symbols
\usepackage{nicefrac}       % compact symbols for 1/2, etc.
\usepackage{microtype}      % microtypography
\usepackage{xcolor}
\usepackage[flushleft]{threeparttable}
\usepackage{color}
\usepackage{mathtools}
\usepackage{amssymb}

\usepackage{url}
\usepackage{float}
\usepackage{graphicx}
\usepackage{multirow}
\usepackage{xspace}
\usepackage{natbib}
\setcitestyle{numbers,square}
\usepackage{enumitem}
\usepackage[font=small]{caption}
\usepackage{diagbox}
\usepackage{wrapfig}
\usepackage{amsmath,amsthm,amssymb}
\newtheorem{theorem}{Theorem}[section]
\newtheorem{lemma}[theorem]{Lemma}
\newtheorem{corollary}[theorem]{Corollary}

\newtheorem{proposition}[theorem]{Proposition}
\newtheorem{remark}[theorem]{Remark}
\newtheorem{assumption}{Assumption}

\providecommand{\norm}[1]{\left\lVert#1\right\rVert}

\providecommand{\E}{{\mathbb E}}
\providecommand{\E}[1]{{\mathbb E}\left.#1\right. }        %expectation
\providecommand{\Eb}[1]{{\mathbb E}\left[#1\right] }       %expectation, with brackets
\renewcommand{\gg}{\mathbf{g}}

\renewcommand{\ll}{\mathbf{l}}

\newenvironment{talign*}
{\csname align*\endcsname}
{\endalign}
\usepackage{algorithm}
\usepackage[noend]{algpseudocode}

\errorcontextlines\maxdimen

\makeatletter
\newcommand*{\algrule}[1][\algorithmicindent]{\makebox[#1][l]{\hspace*{.5em}\thealgruleextra\vrule height \thealgruleheight depth \thealgruledepth}}%
\newcommand*{\thealgruleextra}{}
\newcommand*{\thealgruleheight}{.75\baselineskip}
\newcommand*{\thealgruledepth}{.25\baselineskip}

\newcount\ALG@printindent@tempcnta
\def\ALG@printindent{%
	\ifnum \theALG@nested>0% is there anything to print
	\ifx\ALG@text\ALG@x@notext% is this an end group without any text?
	\else
		\unskip
		\addvspace{-1pt}% FUDGE to make the rules line up
		\ALG@printindent@tempcnta=1
		\loop
		\algrule[\csname ALG@ind@\the\ALG@printindent@tempcnta\endcsname]%
		\advance \ALG@printindent@tempcnta 1
		\ifnum \ALG@printindent@tempcnta<\numexpr\theALG@nested+1\relax% can't do <=, so add one to RHS and use < instead
			\repeat
		\fi
	\fi
}%
\usepackage{etoolbox}
\patchcmd{\ALG@doentity}{\noindent\hskip\ALG@tlm}{\ALG@printindent}{}{\errmessage{failed to patch}}
\makeatother

\newbox\statebox
\newcommand{\myState}[1]{%
	\setbox\statebox=\vbox{#1}%
	\edef\thealgruleheight{\dimexpr \the\ht\statebox+1pt\relax}%
	\edef\thealgruledepth{\dimexpr \the\dp\statebox+1pt\relax}%
	\ifdim\thealgruleheight<.75\baselineskip
		\def\thealgruleheight{\dimexpr .75\baselineskip+1pt\relax}%
	\fi
	\ifdim\thealgruledepth<.25\baselineskip
		\def\thealgruledepth{\dimexpr .25\baselineskip+1pt\relax}%
	\fi
	\State #1%
	\def\thealgruleheight{\dimexpr .75\baselineskip+1pt\relax}%
	\def\thealgruledepth{\dimexpr .25\baselineskip+1pt\relax}%
}
\usepackage{amssymb}
\usepackage{amsmath}
\usepackage{amssymb}
\usepackage{mathtools}
\usepackage{amsthm}
\usepackage{color}
\usepackage[textsize=tiny]{todonotes}

\begin{document}

% \title{Local rank, global rating: a privacy-preserving federated recommendation system with the help of set function.}
% \title{FedRec+: Privacy-Enhancing Heterogeneous Federated Recommendation System with Optimal Aggregation}
\title{FedRec+: Enhancing Privacy and Addressing Heterogeneity in Federated Recommendation Systems}
\author{Lin Wang, Zhichao Wang, Xi Leng, Xiaoying Tang
\thanks{The authors are with the School of Science and Engineering, The Chinese University of Hong Kong (Shenzhen), Shenzhen 518172, China (email: {linwang1,zhichaowang, xileng}@link.cuhk.edu.cn, tangxiaoying@cuhk.edu.cn).
}
}
% \author{some
% \thanks{The authors are with the School of Science and Engineering, The Chinese University of Hong Kong (Shenzhen), Shenzhen 518172, China (email: {xxxxxxx}@link.cuhk.edu.cn, xxxxxxx@cuhk.edu.cn).
% }
% }
% \markboth{Journal of \LaTeX\ Class Files,~Vol.~18, No.~9, September~2020}%
% {How to Use the IEEEtran \LaTeX \ Templates}

\maketitle

\begin{abstract}
Preserving privacy and reducing communication costs for edge users pose significant challenges in recommendation systems. Although federated learning has proven effective in protecting privacy by avoiding data exchange between clients and servers, it has been shown that the server can infer user ratings based on updated non-zero gradients obtained from two consecutive rounds of user-uploaded gradients. Moreover, federated recommendation systems (FRS) face the challenge of heterogeneity, leading to decreased recommendation performance.
In this paper, we propose FedRec+, an ensemble framework for FRS that enhances privacy while addressing the heterogeneity challenge. FedRec+ employs optimal subset selection based on feature similarity to generate near-optimal virtual ratings for pseudo items, utilizing only the user's local information. This approach reduces noise without incurring additional communication costs. Furthermore, we utilize the Wasserstein distance to estimate the heterogeneity and contribution of each client, and derive optimal aggregation weights by solving a defined optimization problem. Experimental results demonstrate the state-of-the-art performance of FedRec+ across various reference datasets.

% Preserving privacy and reducing the communication cost for edge users is a highly common yet challenging task in recommendation systems. Though federated learning has acted as an effective way to protect privacy for decentralized learning that avoids the data exchange among clients and server, it is shown the server can identify items the user has interacted with according to the updated non-zero gradients and further infer the user ratings as long as obtaining the user uploaded gradients in two consecutive rounds. Additionally, FRS also inherited the heterogeneity challenge in FL, resulting in reduced recommended performance.
% In this paper, we propose an ensemble framework, FedRec+, servers to protect the privacy of FRS while addressing the heterogeneity challenge. In particular, FedRec+ first utilizes optimal subset selection via feature similarity, which generates near-optimal virtual rates for pseudo items based only on the user's local information, thus reducing the noise without additional communication cost. Then, we use Wasserstein distance to estimate the heterogeneity and contribution of each client and derive the optimal aggregation weights by solving a defined optimization problem.
% We perform experiments demonstrating the state-of-the-art performance of this method
% across a variety of reference datasets.

\end{abstract}

% \begin{IEEEkeywords}
% Federated learning, federated recommendation system, privacy-preserving, decentralized optimization.
% \end{IEEEkeywords}

 \section{Introduction}

Recommender systems have experienced significant advancements in recent years, enabling personalized recommendations for users \cite{yang2020federated}. However, traditional centralized recommender systems raise concerns about privacy leakage and data integration limitations, as they rely on a central server to store user data \cite{rendle2012factorization,mnih2007probabilistic}.
On the other hand, federated learning (FL) is a distributed learning scheme that ensures privacy preservation by allowing participants to collaboratively train a machine learning model without sharing data \cite{mcmahan2017communication}.
The combination of federated learning and recommendation systems gives rise to federated recommendation systems (FRS), offering a promising solution for privacy-preserving recommendations~\cite{sun2022survey}.

% Therefore, combining federated learning with recommendation systems (FRS) becomes a promising solution for privacy-preserving recommendation systems~\cite{sun2022survey}.

% To address these challenges, federated learning, a privacy-preserving distributed learning scheme, has been proposed \cite{mcmahan2017communication}. This has led to the emergence of federated recommendation systems (FRS) as a promising solution \cite{sun2022survey}.

FRS addresses privacy and data security concerns by decentralizing the recommendation process. User data remains localized on individual devices or servers, and models are trained locally without sharing data. This decentralized approach enhances user privacy and fosters trust. Various approaches, such as federated matrix factorization~\cite{ammad2019federated,liu2021fedct}, federated collaborative filtering~\cite{dolui2019towards,hu2020locality}, and federated deep learning~\cite{minto2021stronger}, distribute the training process across each local parity and aggregate gradients on a central server.

However, privacy preservation remains a major challenge in FRS. Although data decentralization reduces privacy risks compared to conventional data-center training, transmitted gradients between parties can still leak user privacy \cite{wu2021fedgnn}. To address this, various privacy protection mechanisms, including pseudo items~\cite{lin2020fedrec}, homomorphic encryption~\cite{chai2020secure,lin2021fr}, secret sharing~\cite{lin2021fr}, and differential privacy~\cite{dolui2019towards,wu2021fedgnn}, have been incorporated into FRS. Pseudo-item method, in particular, has gained attention due to its low computation and communication costs. By uploading gradients of both interacted and randomly sampled unrated items, Pseudo items prevent the server from inferring user interactions, as shown in Figure~\ref{illustration of pseudo}. However, existing pseudo-item methods suffer from limitations such as introducing significant noise or imposing high communication burdens \cite{lin2020fedrec,liang2021fedrec++}.

Another challenge in FRS is the heterogeneity across local datasets and models, which complicates the aggregation of local recommendations into a coherent global recommendation~\cite {karimireddy2020scaffold}. 

Therefore, in this work, we are primarily interested in addressing two challenges in FRS: \emph{(1) Design an effective pseudo items method that is low noise as well as low communication cost. (2) Design an aggregation algorithm to address the heterogeneity challenge in FRS.}
To effectively address these challenges, we propose an innovative framework called \textbf{FedRec+}, which includes an improved pseudo items method that uses feature similarity to select a subset for virtual rate assignment and an optimal aggregation strategy based on the Wasserstein Distance, as illustrated in Figure~\ref{framework of fedrec+}. FedRec+ effectively preserves client privacy with low computation and communication costs and alleviates the heterogeneity problem in FRS. FedRec+ guarantees convergence with a controllable noise term.

\begin{figure}[!t]
\centering
\includegraphics[width=3.5in]{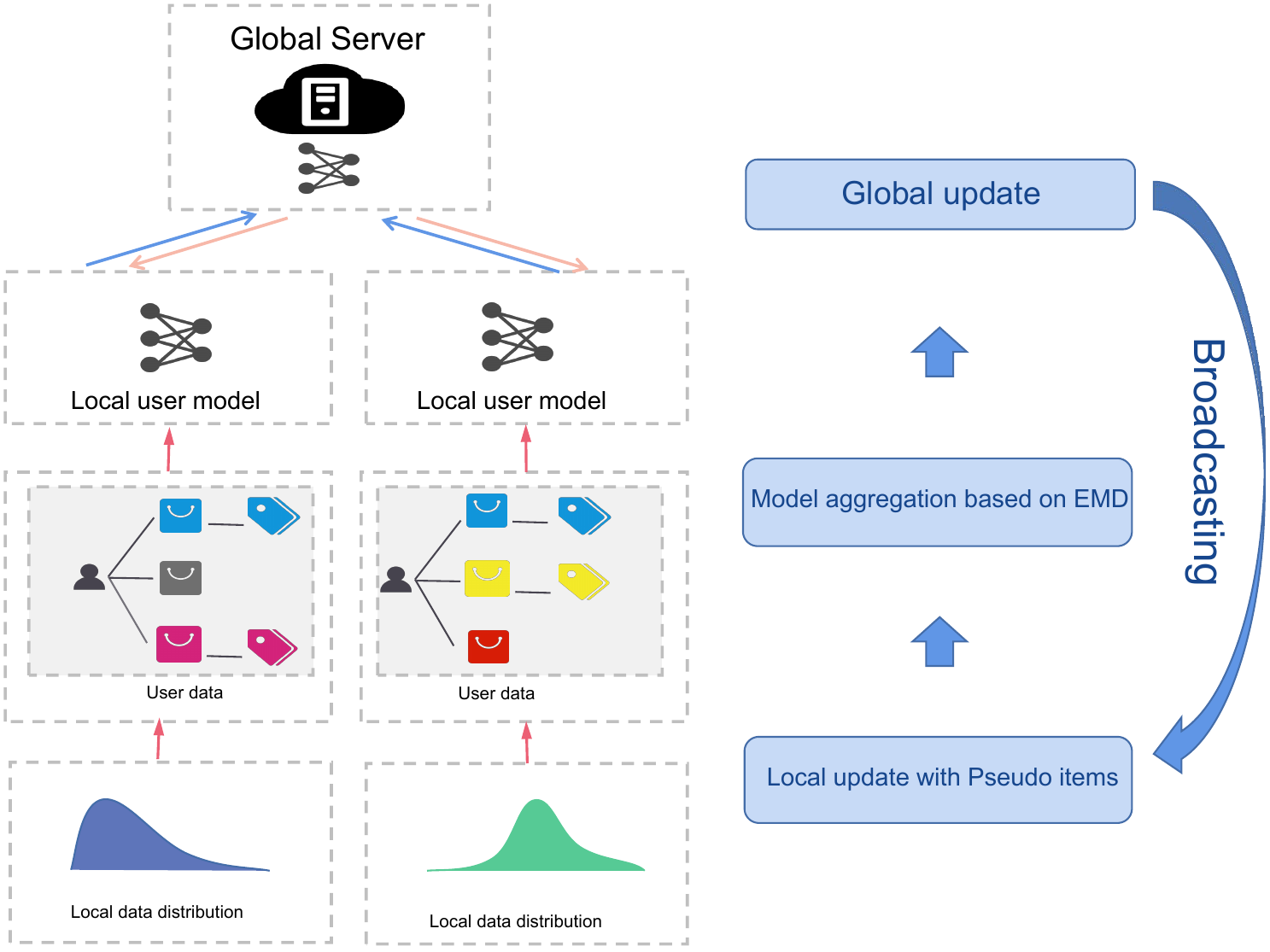}
\caption{\textbf{Framework of FedRec+.} FedRec+ consists of a privacy-preserving component and a dynamic aggregation component. Specifically, (1) it incorporates an enhanced pseudo-items method to safeguard the privacy of interacted items, and (2) it employs an optimal aggregation strategy to address the heterogeneity challenge.}
\label{framework of fedrec+}
% \vspace{-1.5em}
\end{figure}

The contributions of this paper are summarized as follows:
\begin{itemize}[leftmargin=*]
\item We propose FedRec+, a privacy-enhancing FRS algorithm with explicit feedback. 
FedRec+ utilizes feature similarity to generate low-noise pseudo items and incorporates an optimal aggregation strategy derived from the Wasserstein distance between the global and local models to address the statistical heterogeneity problem.

\item We provide a convergence analysis of FedRec+, demonstrating a convergence rate of $\mathcal{O}(\frac{1}{\sqrt{T}}+\frac{1}{T})$. This analysis explicitly highlights the impact of the pseudo-item method and the Wasserstein Distance based aggregation method on the convergence results.

\item We evaluate FedRec+'s performance using public datasets and find that it excels in recommendation performance. Additionally, our ablation study explores the impact of the number of pseudo items.
\end{itemize}

\subsection{Related Work}
\label{relate work}

Several works have explored the use of federated learning in the context of recommendation systems. \cite{ammad2019federated} propose a federated collaborative filtering method for recommendation systems. Other works that follow this line of research include \cite{dolui2019towards, minto2021stronger, hu2020locality}. Additionally, deep learning-based FedRS models have been proposed to leverage user data while ensuring privacy compliance \cite{wu2021fedgnn}.

\paragraph{FRS with Pseudo Items}
To address privacy concerns in FRS, the use of pseudo items has been proposed. \cite{lin2020fedrec} Introduce the concept of pseudo items to protect users' interacted information. However, the vanilla approach of randomly selecting unrated items as pseudo items introduces significant noise.  \cite{liang2021fedrec++} Divide clients into different groups, where one group records the gradients of unrated items uploaded by another group, effectively reducing the noise caused by unrated items. However, this approach requires additional communication and storage costs between users, which can lead to privacy leakage issues \cite{liu2021privacy}. \cite{lin2021fr} Combine secret sharing and pseudo items mechanisms to provide stronger privacy guarantees, while \cite{wu2021fedgnn} combine pseudo items and Local Differential Privacy (LDP) mechanisms to protect user interaction behaviors and ratings in FRS. However, none of these methods effectively address the challenge of large noise from pseudo items while maintaining a low communication cost. In this paper, we propose FedRec+ that leverages each client's own data information to select optimal unrated items, minimizing noise without requiring communication between users.

\paragraph{FRS with Aggregation}
While aggregation algorithms for federated learning (FL) have been extensively studied for various purposes such as convergence acceleration \cite{wang2022delta, chen2021dynamic}, fairness enhancement \cite{wang2023fedeba+}, and robustness improvement \cite{pillutla2022robust}, limited research has been conducted on aggregation algorithms specifically tailored for FRS. \cite{muhammad2020fedfast} Propose FedFast, a federated recommendation model with improved aggregation and update policies. However, there has been no dedicated work addressing the heterogeneity problem in FRS from an aggregation perspective. In this paper, we propose an aggregation algorithm for FRS that utilizes Wasserstein Distance to constrain the objective, effectively tackling the heterogeneity challenge.

\section{SYSTEM MODEL and ALGORITHM}
\label{sec sys architecture}

In this section, we first state the problem setup (Sec~\ref{sec problem setup}), and after explaining the FedRec+ algorithm~\ref{algorithm} (Sec~\ref{sec feature similarity} and Sec~\ref{sec EMD aggregation}), we present our theoretical result along with the underlying assumptions (Sec~\ref{sec analysis}).

\textbf{Notations:} Following the commonly used notations in probabilistic matrix factorization~\cite{koren2009matrix}, the rating of a user $u$ to an item $i$ is calculated as the inner product of their latent feature vectors, i.e., $\hat{r}_{ui} = \textbf{U}_u\textbf{V}_i^{\top}$, where $\textbf{U}_u \in \mathbb{R}^{1\times d}$ and $\textbf{V}_i \in \mathbb{R}^{1\times d}$ are the latent feature vectors of user $u$ and item $i$, respectively.
% We use $u$ and $i$ to denote the user and item, respectively. The latent feature vectors of user $u$ and item $i$ are denoted as $\textbf{U}_u \in \mathbb{R}^{1\times d}$ and $\textbf{V}_i \in \mathbb{R}^{1\times d}$, where the latent dimension $d$ represents the implicit characteristics of the item. The rating of user $u$ and item $i$ is calculated as $\hat{r}_{ui} = \textbf{U}_u\textbf{V}_i^{\top}$. 
The ground-truth rating of item $i$ by user $u$ is denoted as $r_{ui}$. The sets of rated and unrated items for user $u$ are represented as $\mathcal{I}_u$ and $\mathcal{I}_u^{\prime}$, respectively. 
% In each round, the client $u$ calculates the gradients for the rated and unrated items, $\nabla \textbf{V}_i, i\in \mathcal{I}_u \cup \mathcal{I}_u^\prime$, and uploads these gradients to the server. 
The local and global learning rates are denoted as $\eta_L$ and $\eta$, respectively. $b \in [0, B]$ and $k \in [0, K] $ are local batch and local epoch respectively. Boldface characters are used to represent vectors.

% We use $u$ and $i$ to represent the user and item respectively. $\textbf{U}_u \in \mathbb{R}^{1\times d}$ and $\textbf{V}_i \in \mathbb{R}^{1\times d}$ are the latent feature vectors of user $u$ and item $i$, the latent dimension $d$ can be regarded as the item's implicit characteristics, the rating of a user $u$ and item $i$ is calculated by $\hat{r}_{ui} = \textbf{U}_u\textbf{V}_i^{\top}$. The true rate of the item is represented by $r_{ui}$. $\mathcal{I}_u$ and $\mathcal{I}_u^{\prime}$ are the set of rated and unrated items of user $u$. In each round, the client $u$ calculates the gradients to the rated items and unrated items, $\nabla \textbf{V}_i, i\in \mathcal{I}_u \cup \mathcal{I}_u^\prime$, and then uploads these gradients to the server. $\eta_L$ and $\eta$ are the local and global learning rates, respectively. We use boldface characters to represent the vector.
\subsection{Problem Setup}
\label{sec problem setup}

\begin{figure}[!t]
\centering
\includegraphics[width=3.5in]{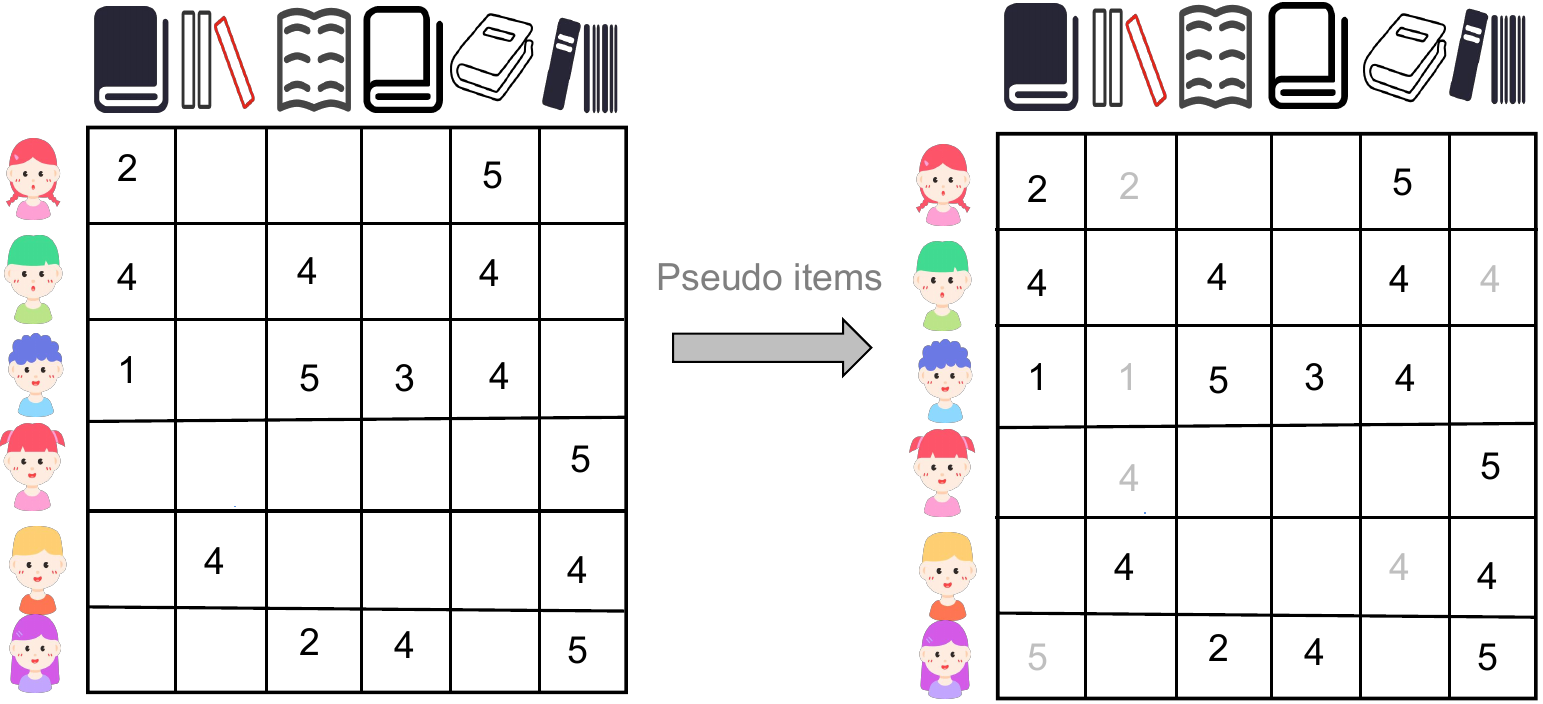}
\caption{\textbf{Illustration of pseudo-item method.} 
To maintain privacy during rating data gradient uploads, the gradients of rated and unrated items are mixed to prevent privacy leaks, safeguarding sensitive information and ensuring privacy protection.
% Note that FedRec uses randomly sampled unrated items for the gradient hybrid filling strategy, but this introduces a large amount of noise; whereas our method uses the learned set function to sample unrated items, which can effectively reduce the noise.
}
\label{illustration of pseudo}
\vspace{-.5em}
\end{figure}

% \begin{figure}[!t]
% \centering
% \includegraphics[width=3.5in]{Fig_illustration/illlustration_selection_network.pdf}
% \caption{\textbf{The illustration of generating the pseudo items.} It values the feature similarity between the rated items and selected unrated items, then assigns the virtual rates to the unrated item 
% The rated items are treated as an optimal subset (OS), and the neural set function is learned under the OS oracle to select a near-optimal subset from the unrated items to be used for pseudo items. In this way, the noise of pseudo items will be largely reduced since it makes the virtual rates near the true rates.}
% \label{illustration of set function network}
% \end{figure}

Before presenting our approach, we provide an overview of the federated matrix factorization (FedMF) algorithm.

In a recommender system, the goal is to fill in missing values of a rating matrix $\boldsymbol{R} \in \mathbb{R}^{n \times m}$. Matrix factorization (MF) is a widely used approach that decomposes the matrix into two low-rank matrices. The rating $r_{u i}$ that user $u$ gives to item $i$ can be approximated as:
\begin{align}
\hat{r}_{u i}=\boldsymbol{U}_u \boldsymbol{V}_i^{\top},
\end{align}
where $\boldsymbol{V}_i$ represents the latent factors of item $i$, and $\boldsymbol{U}_u$ represents the latent factors of user $u$. The latent factors are learned by minimizing a loss function that incorporates the known ratings and regularization terms:

\begin{align}
\min _{\boldsymbol{V}_i, \boldsymbol{U}_u} \frac{1}{2} \sum_{(u, i) \in \mathcal{I}_{u,i}}\left(r_{u i}-\boldsymbol{U}_u \boldsymbol{V}_i^{\top}\right)^2+\lambda\left(\left\|\boldsymbol{V}_i\right\|_2^2+\left\|\boldsymbol{U}_u\right\|_2^2\right) \, ,
\end{align}

where $\mathcal{I}_{u,i}$ represents the set of user-item pairs with known ratings, and $\lambda$ is the regularization coefficient.
 Stochastic gradient descent is utilized to update each parameter: 

\begin{align}
\boldsymbol{V}_i \leftarrow \boldsymbol{V}_i-\eta_L \cdot\left(\lambda \cdot \boldsymbol{V}_i-e_{u i} \cdot \boldsymbol{U}_u\right) \, ,
\end{align}

\begin{align}
\label{U update equation}
 \boldsymbol{U}_u \leftarrow \boldsymbol{U}_u-\eta_L \cdot\left(\lambda \cdot \boldsymbol{U}_u-e_{u i} \cdot \boldsymbol{V}_i\right) \, ,
\end{align}

where $e_{u i}=r_{u i}-\boldsymbol{U}_u \boldsymbol{V}_i^{\top}$ is the prediction error, and $\eta_L$ is the local learning rate.

The vanilla FedMF algorithm \cite{ammad2019federated} extends MF to a federated setting. In FedMF, the item latent factors $\left\{\boldsymbol{V}_i\right\}_{i \in \mathcal{I}}$ are stored on the central server, while each user's latent factors $\boldsymbol{U}_u$ are kept on the local party. The training process consists of the following steps, which are repeated until the convergence of model parameters: (1) The local party downloads item $i$'s latent factors $\boldsymbol{V}_i$ from the server. (2) The local party updates the user's latent factors $\boldsymbol{U}_u$ using its private local data $\boldsymbol{r}_u$. (3) The local party computes the gradients of each item's latent factors $\boldsymbol{g}_{u i}=\lambda \cdot \boldsymbol{V}_i-e_{u i} \cdot \boldsymbol{U}_u$ with  $\boldsymbol{r}_u$ and the updated $\boldsymbol{U}_u$. (4) The local party sends $\boldsymbol{g}_{u i}$ to the server. (5) The server aggregates the gradients $\sum_{u \in \mathcal{U}} \boldsymbol{g}_{u i}$ and updates $\boldsymbol{V}_i$.

However, the vanilla FedMF algorithm suffers from privacy leakage due to the transmitted gradients. The server continuously receives the gradients of the item $i$'s latent vector from user $u$ at step $t-1$ and step $t$:

\begin{align}
\boldsymbol{g}_{u,i}^{t-1}=\lambda \cdot \boldsymbol{V}_i^{t-1}-e_{ui}^{t-1} \cdot \boldsymbol{U}_u^{t-1} \, ,
\end{align}

\begin{align}
\label{graient calculation}
\boldsymbol{g}_{u,i}^{t}=\lambda \cdot \boldsymbol{V}_i^{t}-e_{ui}^{t} \cdot \boldsymbol{U}_u^{t} \, ,
\end{align}

where $\boldsymbol{V}_i^{t-1}$ and $\boldsymbol{V}_i^t$ represent the item $i$'s latent factors at step $t-1$ and step $t$ respectively, and $\boldsymbol{U}_u^{t-1}$ and $\boldsymbol{U}_u^t$ represent the user $u$'s latent factors at step $t-1$ and step $t$ respectively. The server also knows the update rule for the user's latent factors:

\begin{align}
\boldsymbol{U}_u^t = \boldsymbol{U}_u^{t-1}+\gamma \cdot \sum_{i\in \mathcal{I}_u}\left(\lambda \cdot \boldsymbol{U}_u^{t-1}-e_{ui}^t\cdot \boldsymbol{V}_i^t\right) \, ,
\end{align}

where $\mathcal{I}_u$ represents the set of items that user $u$ has rated. Combining these equations, the server can solve for the unknown variables, revealing private raw ratings of each user~\cite{lazard2009thirty}.

To address the gradient leakage problem of vanilla FedMF, several secure FedMF algorithms have been proposed. One such algorithm is FedRec \cite{lin2020fedrec}, which introduces a hybrid filling (HF) strategy to randomly sample unrated items and mix them with rated items. The stochastic gradient descent of FedRec is as follows:
\begin{equation}
    \begin{aligned}
        \nabla V^{\mathrm{HF}}(u, i)=\left\{\begin{array}{l}
\left(U_u \cdot V_i^T-r_{u i}\right) U_u+\lambda V_i, y_{u i}=1, \\
\left(U_u \cdot V_i^T-r_{u i}^{\prime}\right) U_u+\lambda V_i, y_{u i}=0.
\end{array}\right.
    \end{aligned}
\end{equation}
where $r_{u i}$ and $r_{u i}^{\prime}$ are the true observed rating and the virtual rating of user $u$ to item $i$, respectively.

While FedRec ensures privacy protection in rating prediction, the random sampling of items in the hybrid filling strategy introduces noise to the recommendation model, leading to potential performance impacts. This serves as the motivation to develop a lossless version of FedRec, which is crucial for practical deployment in real-world applications.

% Although FedRec achieves privacy protection in rating prediction, the randomly sampled items in the hybrid filling strategy introduce some noise to the recommendation model, which inevitably affects the performance. This motivates us to design a lossless version of FedRec, which is critical to be deployed in a real-world application.

\subsection{Feature similarity for pseudo items} 
\label{sec feature similarity}

While FedRec ensures privacy protection in rating prediction, the hybrid filling strategy, which involves randomly sampling items, introduces noise that impacts performance. To address this, we aim to design a low-noise scheme for assigning rates to unrated items.

Inspired by feature selection techniques utilizing feature similarity~\cite{mitra2002unsupervised}, we aim to select items with characteristics most similar to the rated items for assigning virtual rates. Feature similarity enables the exploration of hidden relationships in the feature space among recommended items~\cite{qian2022feature}. For instance, assuming that some unrated items in the dataset share similar features with rated items having a specific score, the virtual scores of these similar unrated items would be close to that specific score. To reduce noise while maintaining privacy protection, we selectively choose pseudo items that closely align with a user's existing ratings for hybrid filling.

To learn user and item features, we employ an encoder. For instance, let $E_r = Encoder(V_i^r)$ represent the feature of a rated item and $E_{un} = Encoder(V_i^{un})$ represent the feature of an unrated item. We calculate the cosine similarity between these features. Considering $E_r = [x_1, x_2, \ldots, x_n]$ and $E_{un} = [y_1, y_2, \ldots, y_n]$, the cosine similarity $\theta$ measures the angle between the two vectors and is defined as follows:

\begin{align}
\label{similarity feature formulation}
\text{Sim}(E_r, E_{un}) = \cos \theta = \frac{\sum_{i=1}^n (x_i \cdot y_i)}{\sqrt{\sum_{i=1}^n x_i^2 \cdot \sum_{i=1}^n y_i^2}}.
\end{align}

By selecting the top-k unrated items with the most similar features to the scored items, we obtain low-noise pseudo items. These virtual rates, based on latent relationships in the item feature space, introduce less noise compared to randomly assigned scores or randomly averaged virtual scores.

\begin{algorithm}[t]
    \caption{\small FedRec+}
    \label{algorithm}
    \begin{algorithmic}[1]
    \Require{Initialize the model parameters $\boldsymbol{V}_i$, $\boldsymbol{U}_u$, global and local learning rate $\eta$ and $\eta_l$, number of training rounds $T$.
    }
    \Ensure{Trained model $\boldsymbol{V}_i$ and $\boldsymbol{U}_u$}
    \For{ $t = 1, \ldots, T$}
      % \myState{Server samples a subset of $S_t$ clients with $|S_t|=n$ and sends current model $V_t$ to these clients.}
            \For{Each client $u \in U$,in parallel }
                \State{Sample top-k unrated items $\mathcal{I}_u^{\prime}$ based on feature similarity~\eqref{similarity feature formulation} and assign them virtual rates.}
                \For{$k=0,\cdots,K-1$}
                \myState{Update the user's latent factors $\boldsymbol{U}_u$ using \eqref{U update equation}, average on both rated and unrated items.}
                \myState{Get gradients $\boldsymbol{g}_{u,i}^{t,k}$ \eqref{graient calculation} from both the selected unrated and rated items  $i \in \mathcal{I} \cup \mathcal{I}^{\prime} $. }
                 \myState{Local update:$\boldsymbol{V}_u^{t,k+1}=\boldsymbol{V}_i^{t,k}-\eta_L\boldsymbol{g}_{u,i}^{t,k}$.}
                \EndFor 
                  \myState{Send $\boldsymbol{\Delta}_t^i= \boldsymbol{V}_i^{t,K}-\boldsymbol{V}_i^{t,0} =-\eta_L\sum_{k=0}^{K-1} \boldsymbol{g}_{u,i}^{t,k}$ to server}
            \EndFor  
            \For{Server}
                % \For{$k=0,\cdot\cdot\cdot,K-1$}
                % \myState{Server update:$x_{t,k+1}^0=x_{t,k}^0-\eta_Lg_{t,k}^0$, where $g_{t,k}^0 = \nabla F_i(x_{t,k}^0,\xi_{t,k}^0)$.}
                % \myState{$\Delta_t^0 = -\eta_L\sum_{k=0}^{K-1} g_{t,k}^0$.}
                % \EndFor 
                % \State{Normalize $\Delta_t^i$ so that the expression is a scaled linear combination of global update and local update:}
                % \State{$\quad \hat{\Delta}_t^i = c_t \frac{\Delta_t^i-\Delta_t^0}{\|\Delta_t^i-\Delta_t^0\|}$, where $c_t$ is a hyper-parameter.}
                \State{Server updates $\boldsymbol{V}_i$ and sends to clients:}
                \State{$\boldsymbol{V}_i^{t+1} = \boldsymbol{V}_i^{t}+\eta \sum_{u\in U} p_u \boldsymbol{\Delta}_t^i $, where $p_u$ is based on \eqref{simplied aggregation probability}.}
            \EndFor
        \EndFor
    \end{algorithmic}
\end{algorithm}

\subsection{Wasserstein Distance for Aggregation}
\label{sec EMD aggregation}

\begin{figure}[!t]
\centering
\includegraphics[width=3.5in]{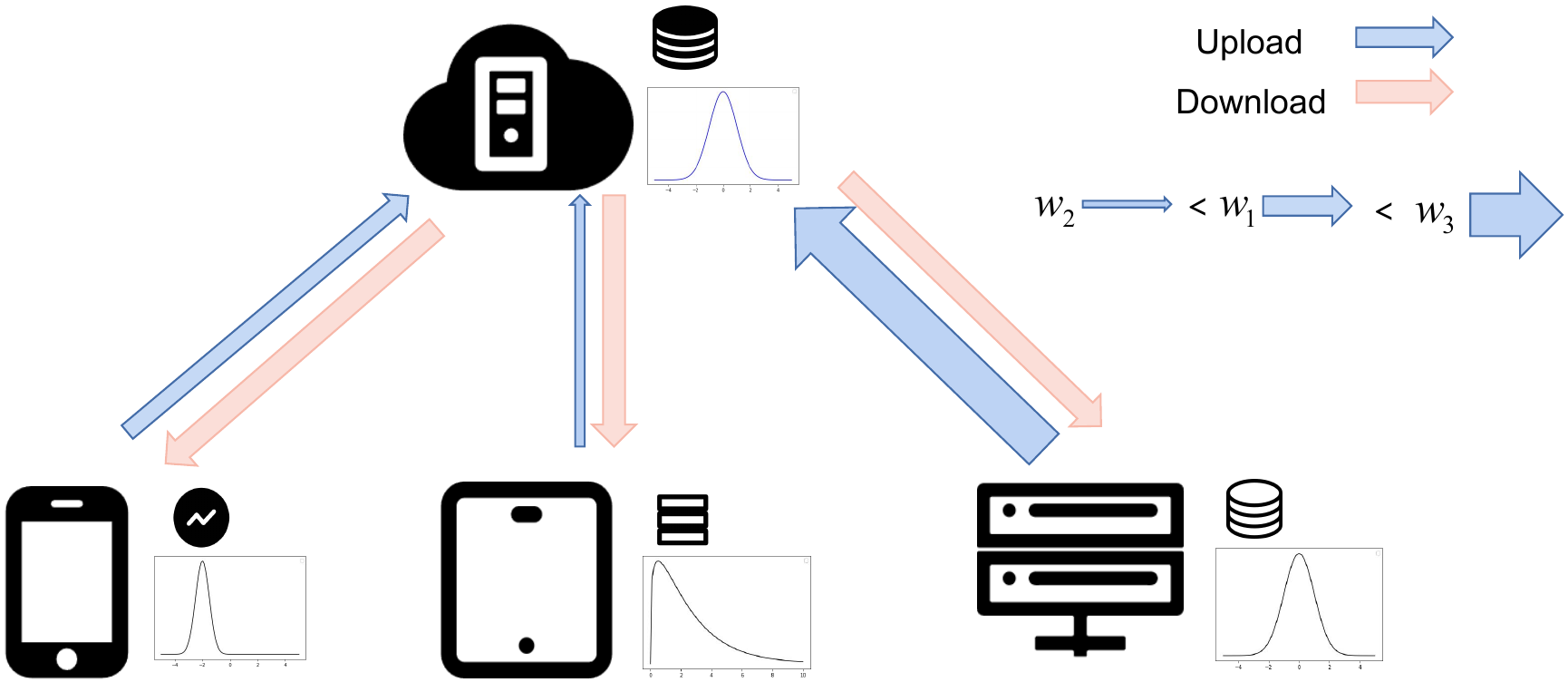}
\caption{\textbf{Illustration of aggregation weights under heterogeneous data.} The aggregation weight is determined inversely proportional to the Wasserstein distance. As the distributions of the local model and global model become closer, the aggregation weight increases accordingly.}
\label{illustration of EMD aggregation}
\vspace{-.5em}
\end{figure}

In this section, we present the derivation of aggregation weights based on Wasserstein distance to address the challenge of statistical heterogeneity in FRS, as depicted in Figure~\ref{illustration of EMD aggregation}.

% Wasserstein distance~\cite{villani2009optimal} is a metric on probability distribution inspired by the problem of optimal transport.
% It measures the dissimilarity between local models and the global model. EMD enables adaptive and dynamic aggregation, taking into account the heterogeneity across local models. We employ EMD as a metric to demonstrate distribution shift because it can measure high-dimensional distributions even in the absence of overlap. Furthermore, in FRS, we can explicitly formulate the distance assuming mean and variance, which naturally arises when the MF method transfers the matrix to the server.
% It is a distance between probability distributions that take geometric information into account. The general Wasserstein distance is defined as:
Wasserstein distance~\cite{villani2009optimal} is a metric on probability distributions inspired by the problem of optimal transport. It is particularly suitable for measuring high-dimensional distributions, even in the absence of overlap. It quantifies the dissimilarity between local models and the global model. 
% To account for the heterogeneity across local models, we employ Earth Mover's Distance (EMD), which enables adaptive and dynamic aggregation. EMD is particularly suitable for measuring high-dimensional distributions, even in the absence of overlap. In the context of FRS, where the Matrix Factorization (MF) method transfers the matrix to the server, we can explicitly formulate the distance assuming mean and variance.
Wasserstein distance of two distribution $\boldsymbol{\mu}$ and $\boldsymbol{\nu}$ is defined as:
\begin{align}
    \boldsymbol{W}_p(\boldsymbol{\mu}, \boldsymbol{\nu})=\inf _{\boldsymbol{\gamma} \in \boldsymbol{\Gamma}(\boldsymbol{\mu}, \boldsymbol{\nu})} \mathbb{E}_{(\boldsymbol{x}, \boldsymbol{y}) \sim \gamma}\left[\|\boldsymbol{x}-\boldsymbol{y}\|_p\right]\, ,
\end{align}
which generally lacks a closed-form solution. However, if we consider the L2-norm as the geometric metric and simplify the problem to a Gaussian distribution, an analytic solution for the distance can be obtained:
\begin{align}
\label{EMD with gaussian distribution}
    d^2=\left\|\boldsymbol{\mu}_1-\boldsymbol{\mu}_2\right\|_2^2+\operatorname{tr}\left(\left(\boldsymbol{\Sigma}_1^{\frac{1}{2}}-\boldsymbol{\Sigma}_2^{\frac{1}{2}}\right)^2\right) \, .
\end{align}
Here, $\boldsymbol{\mu}_1$, $\boldsymbol{\Sigma}_1$ represent the mean and variance of the first distribution, while $\boldsymbol{\mu}_2$, $\boldsymbol{\Sigma}_2$ represent the mean and variance of the second distribution.

Assuming that each client's aggregation weight is denoted by $p_u$, the server aggregates local models $\boldsymbol{w}_{t+1}^1, \cdots, \boldsymbol{w}_{t+1}^{|U|}$ to compute the global model using the following equation:
\begin{align}
    \overline{\boldsymbol{w}} = \sum_{u=1}^{m}p_u \boldsymbol{w}_{t+1}^u \, .
\end{align}
where $|U| = m$ represents the number of users in FRS.

The distance between the aggregated model and the true global model can be formalized as:
\begin{align}
\label{D equation}
    D=\boldsymbol{W}_2\left(p_{\overline{\boldsymbol{w}}}(\boldsymbol{w}), p_{\boldsymbol{w}_g}(\boldsymbol{w})\right) \, ,
\end{align}
where $\boldsymbol{w}_g$ represents the optimal model parameters based on the global distribution, i.e., the distribution of data gathered from all participants. We derive an upper bound for equation~\eqref{D equation}:
\begin{align}
\label{distance upper bound}
D & =\boldsymbol{W}_2\left(p_{\overline{\boldsymbol{w}}}(\boldsymbol{w}), p_{\boldsymbol{w}_g}(\boldsymbol{w})\right) \notag\\
& =\inf _{\boldsymbol{\gamma} \in \boldsymbol{\Gamma}\left(p_{\overline{\boldsymbol{w}}}, p_{\boldsymbol{w}_g}\right)} \mathbb{E}_{(\boldsymbol{x}, \boldsymbol{y}) \sim \gamma}\left[\|\boldsymbol{x}-\boldsymbol{y}\|_2\right] \notag\\
& \leq m\sum_{u=1}^{m} p_u^2 \inf _{\gamma_u \in \boldsymbol{\Gamma}\left(p_{\boldsymbol{w}_{K}^u}, p_{\boldsymbol{w}_g}\right)} \mathbb{E}_{(\boldsymbol{x}_u, \boldsymbol{y}_u) \sim \gamma_u}\left[\|\boldsymbol{x}_u-\boldsymbol{y}_u\|_2\right] \, .
\end{align} 
% where we split $\|\boldsymbol{x}-\boldsymbol{y}\|^1$ as $\left\|\sum_{k=1}^{m} p_k\left(\boldsymbol{x}_k-\boldsymbol{y}_k\right)\right\|_2$, and each pair $\left(\boldsymbol{x}_k, \boldsymbol{y}_k\right)$ is supported on $\boldsymbol{\Gamma}\left(p_{\boldsymbol{w}_{n+1}^k}, p_{\boldsymbol{w}_g}\right)$, then the inequality depends on  Cauchy-Schwarz inequality and the fact $\boldsymbol{w}_{K}^k$ is mutual independent. The upper bound is obvious, since by optimal transport, comprehensively considering all mounds is better than the sum of separate consideration.
In the above equation, we split $\|\boldsymbol{x}-\boldsymbol{y}\|_2$ as $\left\|\sum_{u=1}^{m} p_k\left(\boldsymbol{x}_k-\boldsymbol{y}_k\right)\right\|_2$, where each pair $\left(\boldsymbol{x}_k, \boldsymbol{y}_k\right)$ is supported on $\boldsymbol{\Gamma}\left(p_{\boldsymbol{w}_{K}^u}, p_{\boldsymbol{w}_g}\right)$. The inequality relies on the Cauchy-Schwarz inequality and the independence of $\boldsymbol{w}_{K}^u$. 
By bounding the distance, we establish an upper bound for the gap between the aggregated model and the true global model. Consequently, minimizing the above upper bound effectively approximates the objective of minimizing equation \eqref{D equation}.
% The upper bound is intuitive, as optimal transport considers all mounds comprehensively, which is superior to the sum of separate considerations.

\begin{lemma}
\label{lemma gaussian}
% Let $\xi_{i,k}$ be a sample from a local dataset uniformly at random. For a sufficiently large batch size $B$, the finite-dimensional vector $g = {g_1, g_2, \cdots, g_n }$ converges to a joint distribution approximately according to the Central Limit Theorem, where $g_i=\frac{1}{B}\nabla_w F(w_i, \xi_{i,k})$ for $i\in S_t$.
% This implies that, with mini-batch stochastic gradient descent, the sum of all local updates converges to a Gaussian distribution.
Let $\xi_{u,k}$ be a sample from a local dataset uniformly at random. For a sufficiently large batch size $B$, the finite-dimensional vector $g = {g_1, g_2, \cdots, g_K }$ converges to a joint distribution approximately according to the Central Limit Theorem, where $g_k=\frac{1}{B}\sum_{b=1}^{B}\nabla_w F(w_k, \xi_{k,b})$ for $k\in \{1,2,\cdots, K\}$. This implies that, with mini-batch stochastic gradient descent, the sum of all local updates converges to a Gaussian distribution.
\end{lemma}

\begin{proof}
    For any constant of the local epoch $K$, we can rewrite gradient vector $\boldsymbol{g}$ as 
    \begin{align}
        \boldsymbol{g}=\frac{1}{B} \sum_{b=1}^B\left(\nabla_{\boldsymbol{w}} F\left(\boldsymbol{w}_1, \boldsymbol{\xi}_{1,b}\right), \cdots, \nabla_{\boldsymbol{w}} F\left(\boldsymbol{w}_K, \boldsymbol{\xi}_{K, b}\right)\right) \, ,
    \end{align}
    let $\tilde{\boldsymbol{g}}_b=\left(\nabla_{\boldsymbol{w}} F\left(\boldsymbol{w}_1, \boldsymbol{\xi}_{1, b}\right), \cdots, \nabla_{\boldsymbol{w}} F\left(\boldsymbol{w}_K, \boldsymbol{\xi}_{K, b}\right)\right)$, then we have:
    \begin{align}
        \boldsymbol{g}=\frac{1}{B} \sum_{k=1}^B \tilde{\boldsymbol{g}}_b \, .
    \end{align}
    As long as the gradient norm is upper bounded and $K$ is finite, $\tilde{\boldsymbol{g}}_b$ follows some complex distribution with bounded covariance matrix. Since $\boldsymbol{\xi}_{k, b}$  is sampled independently from the same distribution, $\boldsymbol{g}$ is the mean vector of $\tilde{\boldsymbol{g}}_1, \cdots, \tilde{\boldsymbol{g}}_B$, which are independent and identically distributed (i.i.d.) random vectors. Therefore, according to the Central Limit Theorem, $\boldsymbol{g}$ converges to $\mathcal{N}(\boldsymbol{\mu}, \boldsymbol{\Sigma})$ in distribution.
    % then as long as the gradient norm is upper bounded and $n$ is finite, $\tilde{\boldsymbol{g}}_b$ subject to some complex distribution with bounded covariance matrix. As $\boldsymbol{\xi}_{i, b}$ is sampled from the same distribution independently, $\boldsymbol{g}$ is the mean vector of $\tilde{\boldsymbol{g}}_1, \cdots, \tilde{\boldsymbol{g}}_B$, which are iid random vectors. Therefore, $\boldsymbol{g}$ converges to $\mathcal{N}(\boldsymbol{\mu}, \boldsymbol{\Sigma})$ in distribution due to the Central Limit Theory.
\end{proof}
% Lemma~\ref{lemma gaussian}, implies that, with mini-batch stochastic gradient descent, the sum of all local updates converges to a Gaussian distribution, i.e.,
Lemma~\ref{lemma gaussian} implies that, with mini-batch stochastic gradient descent, the sum of all local updates converges to a Gaussian distribution, i.e.:
\begin{align}
\label{gradient gaussian distribution converge}
    \overline{\boldsymbol{g}}=\boldsymbol{w}_{K}-\boldsymbol{w}_1=\sum_{k=1}^K \boldsymbol{g}_k=\mathbf{1}^T \boldsymbol{g} \, ,
\end{align}
where 
% $\overline{\boldsymbol{g}}$ is a linear transformation of a joint Gaussian vector, thus it conforms to Gaussian distribution. As Gaussian distribution is determined by two variables: the mean vector and the covariance matrix, our next target is to estimate these variables. 
$\overline{\boldsymbol{g}}$ is a linear transformation of a joint Gaussian vector, thus it conforms to a Gaussian distribution. As the Gaussian distribution is determined by the mean vector and the covariance matrix, our next goal is to estimate these variables.

% According to \eqref{gradient gaussian distribution converge}, we know that the total gradient distribution is approximated by the sum of independent random vectors, i.e., $\overline{\boldsymbol{g}}=\boldsymbol{w}_{K}-\boldsymbol{w}_1=\sum_{k=1}^K \eta \boldsymbol{g}_k=\eta_L \mathbf{1}^T \boldsymbol{g}$, where $\eta_L$ is the local learning rate, hence the corresponding parameters can be estimated as $\boldsymbol{\mu}=\eta_L \sum_{k=1}^K \mathbb{E}\left[\boldsymbol{g}_k\right] \doteq K \eta_L \mathbb{E}\left[\boldsymbol{g}_1\right]$ and $\boldsymbol{\Sigma}=\eta_L^2 \sum_{k=1}^K \boldsymbol{\Sigma}_k \doteq K \eta_L^2 \boldsymbol{\Sigma}_1$, here index 1 represents a random user and $E[\boldsymbol{g}_1]$ and $\boldsymbol{\Sigma}_1$ represent the average gradients and average variance of client 1. In particular, based on the relationship between covariance matrix, correlation matrix and the mean vector, we can get $\Sigma_i = \E[g_{i}g_i^T]-\E[g_i]\E[g_i^T]$, where $g_i=\frac{1}{B}\sum_{k=1}^B\nabla F(x,\xi_{i,k})$.
% Besides, implied by Lemma~\ref{lemma gaussian}, the global gradient distribution converges to $\mathcal{N}(\boldsymbol{\mu}_g,\frac{\boldsymbol{\Sigma}_g}{B})$. Then according to \eqref{EMD with gaussian distribution} and \eqref{distance upper bound}, to minimize the distance between aggregated model and global model, we can build an optimization problem as

Based on \eqref{gradient gaussian distribution converge}, we know that the total gradient distribution is approximated by the sum of independent random vectors, i.e. $\overline{\boldsymbol{g}}=\boldsymbol{w}_{K}-\boldsymbol{w}_1=\sum_{k=1}^K \eta \boldsymbol{g}_k=\eta_L \mathbf{1}^T \boldsymbol{g}$, where $\eta_L$ is the local learning rate. Therefore, the corresponding parameters can be estimated as  $\boldsymbol{\mu}=\eta_L \sum_{k=1}^K \mathbb{E}\left[\boldsymbol{g}_k\right] \doteq K \eta_L \mathbb{E}\left[\boldsymbol{g}_1\right]$ and $\boldsymbol{\Sigma}=\eta_L^2 \sum_{k=1}^K \boldsymbol{\Sigma}_k \doteq K \eta_L^2 \boldsymbol{\Sigma}_1$. Here, index 1 represents a random user, and $E[\boldsymbol{g}_1]$ and $\boldsymbol{\Sigma}_1$ represent the average gradients and average variance of client 1. In particular, based on the relationship between the covariance matrix, correlation matrix, and the mean vector, we can obtain $\Sigma_k = \E[g_{k}g_k^T]-\E[g_k]\E[g_k^T]$, where $g_k=\frac{1}{B}\sum_{b=1}^B\nabla F(x,\xi_{k,b})$.

Furthermore, implied by Lemma~\ref{lemma gaussian}, the global gradient distribution converges to $\mathcal{N}(\boldsymbol{\mu}_g,\frac{\boldsymbol{\Sigma}_g}{B})$. Therefore, to minimize the distance between the aggregated model and the global model, we can formulate an optimization problem as follows:
\begin{equation}
    \begin{aligned}
        \label{optimization wrt p}
\min _{p_k} \  & D = \sum_{u=1}^{m}p_u^2\left( \left\|\eta \boldsymbol{\mu}_g - \eta_L K\boldsymbol{\mu}_u\right\|^2+\operatorname{tr}\left(\boldsymbol{M}^2\right) \right)\\
\text { s.t. } & \boldsymbol{M}=\left(\frac{\eta^2 \boldsymbol{\Sigma}_g}{B}\right)^{\frac{1}{2}}-\left(\frac{K^2 \eta_L^2 \boldsymbol{\Sigma}_u}{B}\right)^{\frac{1}{2}}, \sum_{u=1}^m p_u = 1, p_u \geq 0 \, .
    \end{aligned}
\end{equation}

\begin{proposition}
    The optimal server aggregation weights that minimize the distribution distance between the aggregated model and the ideal global model, using Wasserstein Distance, are given by:
    \begin{align}
        p_u^* = \frac{\frac{1}{\| K\eta_L\boldsymbol{\mu}_u - \eta \boldsymbol{\mu}_g\|^2+\operatorname{tr}\left(\boldsymbol{M}^2\right)}}{\sum_{u=1}^m \frac{1}{\|K\eta_L\boldsymbol{\mu}_u - \eta \boldsymbol{\mu}_g\|^2+\operatorname{tr}\left(\boldsymbol{M}^2\right)} }.
    \end{align}
\end{proposition}

\begin{proof}
It can be seen that \eqref{optimization wrt p} is a convex optimization problem,
which we use the Karush–Kuhn–Tucker (KKT) conditions to solve. Introducing Lagrange multipliers $\lambda_u \in \mathbb{R}$ for the inequality constraints $p_u\geq 0$, and a multiplier $\nu \in \mathbb{R}$ for the equality constraint $\sum_u p_u = 1$, we have
\begin{align}
\label{eq condition}
&\lambda_u \geq 0, \lambda_u p_u=0,  \sum_{u \in[m]} p_u=1, \quad p_u \geq 0 , \notag \\
 & 2 \left(\left\|\eta \boldsymbol{\mu}_g - \eta_L K\boldsymbol{\mu}_u\right\|^2+\operatorname{tr}\left(\boldsymbol{M}^2\right) \right) p_u-\lambda_u+\nu=0.
\end{align}
Since $\sum_{u \in[m]} p_u=1$, there exists $u_0$ such that $p_{u_0}>0$. Thus we have $\lambda_{u_0}=0$, which yields $\nu=-2\left(\left\|\eta \boldsymbol{\mu}_g - \eta_L K\boldsymbol{\mu}_u\right\|^2+\operatorname{tr}\left(\boldsymbol{M}^2\right)\right) p_{u_0}<0$. Therefore, $p_u>0$ always holds because if $p_u=0$, it leads to $2\left(\left\|\eta \boldsymbol{\mu}_g - \eta_L K\boldsymbol{\mu}_u\right\|^2+\operatorname{tr}\left(\boldsymbol{M}^2\right)\right) p_u-\lambda_u+\nu<0$ which violates the condition in \eqref{eq condition}. As a result, we have $\lambda_u=0, \forall u \in[m]$. Furthermore,
\begin{align}
\label{mid term 1}
    p_u=-\frac{\nu}{2\left(\left\|\eta \boldsymbol{\mu}_g - \eta_L K\boldsymbol{\mu}_u\right\|^2+\operatorname{tr}\left(\boldsymbol{M}^2\right)\right)}, \quad \forall u \in[m] .
\end{align}
By plugging \eqref{mid term 1} into $\sum_{u \in[m]} p_u=1$, we have
\begin{align}
\label{mid term 2}
    \nu=-\frac{2}{\sum_u 1 /\left(\left\|\eta \boldsymbol{\mu}_g - \eta K\boldsymbol{\mu}_u\right\|^2+\operatorname{tr}\left(\boldsymbol{M}^2\right)\right)} .
\end{align}
Plugging \eqref{mid term 2} back into \eqref{mid term 1} completes the proof.
\end{proof}

However, considering the increased communication traffic and computational complexity introduced by the covariance matrix, we need to simplify the procedure. Note that
\begin{align}
\label{simplied aggregation probability}
    \lim _{B \rightarrow+\infty} p^{*}_u=\frac{\frac{1}{\|K\eta_L\boldsymbol{\mu}_u - \eta \boldsymbol{\mu}_g\|^2}}{\sum_u \frac{1}{\|K\eta_L\boldsymbol{\mu}_u - \eta \boldsymbol{\mu}_g\|^2}} \, .
\end{align}
% hence with a sufficient large batch size, we can use ~\eqref{simplied aggregation probability} to estimate the optimal aggregation probability. 
% So far, we have derived an optimal aggregation strategy by using the Wasserstein distance to tackle the heterogeneity challenge in FRS.
Hence, with a sufficiently large batch size, we can use \eqref{simplied aggregation probability} to estimate the optimal aggregation probability. Thus, we have derived an optimal aggregation strategy using the Wasserstein distance to address the challenge of heterogeneity in FRS.

% \begin{align}
% \label{aggregation probability}
%     w_i = \frac{\boldsymbol{\mu}_i^T\boldsymbol{\mu}_g}{\|\boldsymbol{\mu}_i\|^2}, s.t.  \sum_i w_i = 1
% \end{align}

% \begin{algorithm}[t]
%     \caption{\small OptimalSubset Selection}
%     \label{algorithm}
%     \begin{algorithmic}[1]
%     \Require{The rated 
%     }
%     \For{ $t = 1, \ldots, T$}
%             \For{Each client $u \in U$,in parallel }
%                 \State{Sample items $\mathcal{I}_u^{\prime}$ and assign the virtual rates according to OptimalSubset Selection.}
%                 \For{$k=0,\cdots,K-1$}
%                 \myState{Get gradients $g$ from both the selected unrated and rated items  $\nabla V_i, i \in I \cup I^{\prime} $. }
%                  \myState{Local update:$V_u^{t,k+1}=V_u^{t,k}-\eta_Lg_u^{t,k}$, where $g_u^{t,k} = -2(e_{uv}V_i-\lambda U_u)$.}
%                 \EndFor 
%                   \myState{Send $\Delta_t^i= V_u^{t,K}-V_u^{t,0} =-\eta_L\sum_{k=0}^{K-1} g_{t,k}^i$ to server}
%             \EndFor  
%             \For{Server}
%                 \State{Server update:}
%                 \State{$\quad x_{t+1} = x_{t}+\eta \sum_{i\in S_t} w_i \Delta_t^i $, where $w_i$ is based on \eqref{ }.}
%             \EndFor
%         \EndFor
%     \end{algorithmic}
% \end{algorithm}

\section{Theoretical analysis}
\label{sec analysis}
In this section, to ease the theoretical analysis, we redefine some notations: the parameter of the model is $\boldsymbol{x}$ instead of $\boldsymbol{U}$ and $\boldsymbol{V}$, and use index $i\in [1,\cdots,m],k\in [1,K],t\in [0,T]$ to indicate user, local epoch, and communication round. The optimization objective of FRS is formulated as follows:
\begin{align}
    \min_{\boldsymbol{x}\in \mathbb{R}^d} f(\boldsymbol{x}) = \E_{i \sim \mathcal{P}}\left[F_i(\boldsymbol{x})\right] \, ,
\end{align}
where $F_i(\boldsymbol{x}) \triangleq \E_{\xi \sim P_i}[F_i(\boldsymbol{x},\xi)]$. Here, $\mathcal{P}$ represents the overall data distribution of entire client distribution, $\boldsymbol{x} \in \mathbb{R}^d$ is the model parameter, $F_i(\boldsymbol{x})$ represents the local loss function at client $i$ and $P_i$ is the underlying distribution of local dataset at client $i$. In general, $P_i \neq P_j$ if $i \neq j$ due to data heterogeneity. However, the loss function $F(\boldsymbol{x})$ or full gradient $\nabla F(\boldsymbol{x})$ can not be directly computed as the exact distribution of data is unknown in general. Hence, one often consider the following empirical risk minimization (ERM) problem in the form of finite-sum instead:
\begin{align}
    \min_{\boldsymbol{x}\in \mathbb{R}^d} f(\boldsymbol{x}) = \sum_{i\in S_t} p_i F_i(\boldsymbol{x}) \, ,    
\end{align}
where $F_i(\boldsymbol{x})=\frac{1}{|D_i|}\sum_{\xi \in D_i}F_i(\boldsymbol{x}, \xi) $. Here, $S_t$ is the selected client set in each round and $p_i$ is the aggregation weights of clients.

To ease the theoretical analysis of our work, we use the following widely used assumptions:

\begin{assumption}[L-Smooth]
\label{Assumption 1}
There exists a constant $L>0$, such that 
$\norm{\nabla F_i(x)-\nabla F_i(y)} \leq L \norm{x-y},\forall x,y \in \mathbb{R}^d$, and $i = {1,2,\ldots,m}$.
\end{assumption}

\begin{assumption}[Unbiased Local Gradient Estimator and Local Variance]
\label{Assumprion 2}
	Let $\xi_t^i$ be a random local data sample in the round $t$ at client $i$: $\Eb{\nabla F_i(\boldsymbol{x}_t,\xi_t^i)}=\nabla F_i(\boldsymbol{x}_t), \forall i \in [m]$. There exists a constant bound $\sigma_{L}>0$, satisfying $\Eb{ \norm{ \nabla F_i(\boldsymbol{x}_t,\xi_{t}^i)-\nabla F_i(\boldsymbol{x}_t) }^2 }\leq \sigma_L^2$.
\end{assumption}

\begin{assumption}[Bound Gradient Dissimilarity] 
\label{Assumption 3}
% For any sets of weights $\left\{w_i \geq 0\right\}_{i=1}^m, \sum_{i=1}^m w_i=1$, there exists constant $\sigma_G^2 \geq 0$ and  $A\geq0$ that satisfy $\sum_{i=1}^m w_i\left\|\nabla F_i(x)\right\|^2 \leq (A^2+1)\left\|\sum_{i=1}^m w_i \nabla F_i(x)\right\|^2 + \sigma_G^2$. 
For any set of weights $\left\{w_i \geq 0\right\}_{i=1}^m$ with $\sum_{i=1}^m w_i=1$, there exist constants $\sigma_G^2 \geq 0$ and $A\geq0$ such that $\sum_{i=1}^m w_i\left\|\nabla F_i(\boldsymbol{x})\right\|^2 \leq (A^2+1)\left\|\sum_{i=1}^m w_i \nabla F_i(\boldsymbol{x})\right\|^2 + \sigma_G^2$.
\end{assumption}
The above three assumptions are commonly used in both non-convex optimization and FL literature, see e.g.~\cite{karimireddy2020scaffold,yang2021achieving}. 
For Assumption~\ref{Assumption 3}, if all local loss functions are identical, then we have $A=0$ and $\sigma_G = 0$.    

Since there are both rated items and pseudo items, $\nabla F_i(\boldsymbol{x}_{t,k}^i) = \frac{1}{B}\sum_{b \in B}\nabla F(\boldsymbol{x}_{t,k}^i, \xi_{b}) = \alpha  \frac{1}{B^r}\sum_{b^r \in B^r}\nabla F(\boldsymbol{x}_{t,k}^i, \xi_{b^r}) + (1-\alpha)  \frac{1}{B^u}\sum_{b^u \in B^u}\nabla F(\boldsymbol{x}_{t,k}^i, \xi_{b^u}) = \alpha \nabla \overline{F}(\boldsymbol{x}_{t,k}^i) + (1-\alpha) \tilde{F}(\boldsymbol{x}_{t,k}^i)$, where $B^r$ and $B^u$ represent the rated items and unrated items, respectively. $B=B^r\cup B^u$ represents total items. $\alpha = \frac{B^r}{B}$ is the relative ratio of rated items in all user's items.
\begin{assumption}[Gradient Difference Bound]
\label{assumption 4}
In each round, we assume that the gradient of the pseudo item is denoted as $\nabla \tilde{F}(\boldsymbol{x}_{t,k}^i)$, while its true gradient is denoted as $\nabla \overline{F} (\boldsymbol{x}_{t,k}^i)$. The gap of the approximation satisfies the following conditions:
    % The gradient variance of pseudo items $\nabla \tilde{f}(x)$ is bounded as below:
$\E\| \nabla \overline{F}(\boldsymbol{x}_{t,k}^i) -\nabla \tilde{F}(\boldsymbol{x}_{t,k}^i)\|^2 \leq \rho^2$, $\forall i,t,k$.
\end{assumption}

\begin{theorem}[Convergence rate]
\label{convergence theorem}
Under Assumption~\ref{Assumption 1}-~\ref{assumption 4}, and let constant local and global learning rate $\eta_L$ and $\eta$ be chosen such that $\eta_L < min\left(1/(8LK), C\right)$, where $C$ is obtained from the condition that $\frac{1}{2}-10 L^2\frac{1}{m}\sum_{i-1}^m K^2\eta_L^2(A^2+1)(\chi_{p\|w}^2A^2+1) \textgreater C \textgreater 0$, and $\eta \leq 1/(\eta_LL)$. The expected gradient norm of  FedRec+ is bounded as follows: 
    \begin{align}
        \mathop{min} \limits_{t\in[T]} \E\|\nabla f(\boldsymbol{x}_t)\|^2 \leq 2 \left[\chi_{\boldsymbol{w} \| \boldsymbol{p}}^2A^2+1\right] (\frac{f_0-f_*}{c\eta\eta_LK T}+\Phi) \notag\\
        +2 \chi_{\boldsymbol{w} \| \boldsymbol{p}}^2 \sigma_G^2 \, ,
    \end{align}
where $f_0=f(x_0)$, $f_* = f(x_*)$, and 
\begin{equation}
    \begin{aligned}
        &\Phi = \frac{1}{c} \left[
     \frac{5\eta_L^2KL^2}{2}(\sigma_L^2 +6K\sigma_G^2+6K(1-\alpha)^2\rho^2) \right. \notag\\
     & \left.+ \frac{\eta\eta_L L}{2}(\sigma_L^2+(1-\alpha)^2\rho^2) +20L^2K^2(A^2+1)\eta_L^2\chi_{\boldsymbol{w} \| \boldsymbol{p}}^2\sigma_G^2\right] \, .
    \end{aligned}
\end{equation}
% \begin{align}
%     &\Phi = \frac{1}{c} \left[
%      \frac{5\eta_L^2KL^2}{2}(\sigma_L^2 +6K\sigma_G^2+6K(1-\alpha)^2\rho^2) \right. \notag\\
%      & \left.+ \frac{\eta\eta_L L}{2}(\sigma_L^2+(1-\alpha)^2\rho^2) +20L^2K^2(A^2+1)\eta_L^2\chi_{\boldsymbol{w} \| \boldsymbol{p}}^2\sigma_G^2\right] \, .
% \end{align}
where $\chi_{\boldsymbol{w} \| \boldsymbol{p}}^2=\sum_{i=1}^m\left(w_i-p_i\right)^2 / p_i$ represents the chi-square divergence between vectors $\boldsymbol{w}=\left[\frac{1}{m}, \ldots, \frac{1}{m}\right]$ and $ \boldsymbol{p}=\left[p_1, \ldots, p_m\right]$.  Observe that when all clients have uniform data distribution, we have $\boldsymbol{p} = \boldsymbol{w}$ such that $\chi^2 =0$.
\end{theorem}

\begin{corollary}
    Suppose $\eta_L$ and $\eta$ are $\eta_L=\mathcal{O}\left(\frac{1}{\sqrt{T}KL}\right)$ and $\eta=\mathcal{O}\left(\sqrt{Km}\right)$ such that the conditions mentioned above are satisfied. Then for sufficiently large T, the iterates of FedRec+ satisfy:
    \begin{align}
        &\min _{t \in[T]}\left\|\nabla f\left(\boldsymbol{x}_t\right)\right\|^2 
        \leq \mathcal{O}\left(\frac{(f^0-f^*)}{\sqrt{mKT}} \right) \notag \\
        &+\mathcal{O}\left(\frac{\sqrt{m}(\sigma_L^2+(1-\alpha)^2\rho^2)}{2\sqrt{KT}} \right) + \mathcal{O}\left(\frac{20(A^2+1)\chi_{\boldsymbol{w} \| \boldsymbol{p}}^2\sigma_G^2}{T} \right)
        \notag \\
       & + \mathcal{O}\left(\frac{5(\sigma_L^2+6K\sigma_G^2+6K(1-\alpha)^2\rho^2)}{2KT} \right) 
         + 2 \chi_{\boldsymbol{w} \| \boldsymbol{p}}^2 \sigma_G^2 \, .
    \end{align}
\end{corollary}

\begin{remark}[Effects of pseudo items and reweight aggregation]
   The noise error introduced by pseudo items is denoted by $\rho$. It is observed that a larger value of $1-\alpha$ corresponds to a larger noise, implying that using more pseudo items leads to increased noise.
    The non-vanishing term $\chi_{\boldsymbol{w} | \boldsymbol{p}}^2 \sigma_G^2$ represents the aggregation error arising from an unbiased aggregation distribution. In other words, there is always an error term present in the convergence rate as long as the aggregation algorithm exhibits bias.
\end{remark}

\emph{Proof Sketch:}
Using the following two lemmas, we can finish the proof. In particular, based on the Lemma~\ref{objective gap}, we bound the gradient norm of $\nabla f(\boldsymbol{x}_t)$ by the norm of $\nabla \tilde{f}(\boldsymbol{x}_t)$, and then utilize Lemma~\ref{Our local update bound} we derive the upper bound for $\nabla \tilde{f}(\boldsymbol{x}_t)$.

\begin{lemma}[Gradient distance between pseudo items and rated items.]
For any model parameter $\boldsymbol{x}$, the difference between the gradients of $f(\boldsymbol{x})$ and $\tilde{f}(\boldsymbol{x})$ can be bounded as follows:
\begin{align}
    \|\nabla f(\boldsymbol{x})-\nabla \tilde{f}(\boldsymbol{x})\|^2 \leq \chi_{\boldsymbol{w} \| \boldsymbol{p}}^2\left[A^2\|\nabla \tilde{f}(\boldsymbol{x})\|^2+\kappa^2\right] \, ,
\end{align}
$f(x)$ is the true objective with $f(x)=\sum_{i=1}^mw_if_i(x)$ where $\boldsymbol{w}$ is usually average of all clients, i.e., $\boldsymbol{w} = \frac{1}{m}$. $\tilde{f}(x)=\sum_{i=1}^{m}p_if_i(x)$ is the surrogate objective with the reweight aggregation probability $\boldsymbol{p}$.
\label{objective gap}
\end{lemma}
    
\begin{proof}
\begin{equation}
\begin{aligned}
 \nabla f(\boldsymbol{x})-\nabla \tilde{f}(\boldsymbol{x}) 
% =\sum_{i=1}^m\left(w_i-p_i\right) \nabla f_i(\boldsymbol{x}) 
% & =\sum_{i=1}^m\left(w_i-p_i\right)\left(\nabla f_i(\boldsymbol{x})-\nabla \tilde{f}(\boldsymbol{x})\right) \\
 =\sum_{i=1}^m \frac{w_i-p_i}{\sqrt{p}_i} \cdot \sqrt{p}_i\left(\nabla f_i(\boldsymbol{x})-\nabla \tilde{f}(\boldsymbol{x})\right) \, .
\end{aligned}
\end{equation}

Applying Cauchy-Schwarz inequality, it follows that
\begin{equation}
\begin{aligned}
&\|\nabla f(\boldsymbol{x})-\nabla \tilde{f}(\boldsymbol{x})\|^2   \\ 
&  \leq \left[\sum_{i=1}^m \frac{\left(w_i-p_i\right)^2}{p_i}\right]\left[\sum_{i=1}^m p_i\left\|\nabla F_i(x)-\nabla \tilde{f}(\boldsymbol{x})\right\|^2\right] \\
 & \leq \chi_{\boldsymbol{w} \| \boldsymbol{p}}^2\left[A^2\|\nabla \tilde{f}(\boldsymbol{x})\|^2+\sigma_G^2\right] \, ,
\end{aligned}
\end{equation}
where the last inequality uses Assumption~\ref{Assumption 3}.
Note that
\begin{equation}
\begin{aligned}
&\|\nabla f(\boldsymbol{x})\|^2  \leq 2\|\nabla f(\boldsymbol{x})-\nabla \tilde{f}(\boldsymbol{x})\|^2+2\|\nabla \tilde{f}(\boldsymbol{x})\|^2 \\
& \leq 2\left[\chi_{\boldsymbol{w} \| \boldsymbol{p}}^2A^2+1\right]\|\nabla \tilde{f}(\boldsymbol{x})\|^2+2 \chi_{\boldsymbol{p} \| \boldsymbol{w}^2}^2 \sigma_G^2 \, .
\end{aligned}
\end{equation}

As a result, we obtain:
\begin{align}
\min _{t \in[T]} &\left\|\nabla f\left(\boldsymbol{x}_t\right)\right\|^2  \leq \frac{1}{\tau} \sum_{t=0}^{T-1}\left\|\nabla f\left(\boldsymbol{x}_t\right)\right\|^2 \notag \\
& \leq 2\left[\chi_{\boldsymbol{w} \| \boldsymbol{p}}^2A^2+1\right] \frac{1}{\tau} \sum_{t=0}^{T-1}\left\|\nabla \tilde{f}\left(\boldsymbol{x}_t\right)\right\|^2+2 \chi_{\boldsymbol{w} \| \boldsymbol{p}}^2 \sigma_G^2 \notag\\
& \leq 2\left[\chi_{\boldsymbol{w} \| \boldsymbol{p}}^2A^2+1\right] \epsilon_{\mathrm{opt}}+2 \chi_{\boldsymbol{w} \| \boldsymbol{p}}^2 \sigma_G^2 \, ,
\label{global convergence}
\end{align}
where $\epsilon_{\mathrm{opt}}$ denotes the optimization error.
\end{proof}

\begin{lemma}[Local updates bound.] 
\label{Our local update bound}
For any step-size satisfying $\eta_L \leq \frac{1}{8LK}$, we can have the following results:
\begin{equation}
    \begin{aligned}
        \E\|x_{t,k}^i-x_t\|^2 \leq 5K\left(\eta_L^2\sigma_L^2+6K\eta_L^2(1-\alpha)^2\rho^2+6K\eta_L^2\sigma_G^2\right) \notag \\
        + 30K^2(A^2+1)\eta_L^2\|\nabla f(x_t)\|^2 \, .
    \end{aligned}
\end{equation}
\end{lemma}

\begin{proof}

% \begin{align}
% &\E_t\|x_{t,k}^i-x_t\|^2 =\E_t\|x_{t,k-1}^i-x_t-\eta_Lh_{t,k-1}^t\|^2   \notag\\
% &\E_t\|x_{t,k-1}^i-x_t-\eta_L ( (1-\alpha) \nabla F(x_{t,k-1}^i;\xi_{b^u}) \notag \\
% & + \alpha \nabla F(x_{t,k-1}^i;\xi_{b^r}) 
% -(1-\alpha)\nabla \overline{F}_i(x_{t,k-1}^i)-\alpha \nabla \overline{F}_i(x_{t,k-1}^i) \notag \\
% & +(1-\alpha)\nabla \tilde{F}_i(x_{t,k-1}^i)
%     -(1-\alpha)\nabla \tilde{F}_i(x_{t,k-1}^i)
%     +\nabla F_i(x_{t,k-1}^i) \notag \\
% &    -\nabla F_i(x_t) + \nabla F_i(x_t)  ) \|^2  \notag \\
%  & \leq (1+\frac{1}{2K-1})\E_t\|x_{t,k-1}^i-x_t\|^2 \notag \\
%  &+ (1-\alpha)^2\eta_L^2\E\|\nabla F(x_{t,k-1}^i;\xi_{b^u}) 
%  -\nabla \tilde{F}_i(x_{t,k-1}^i) \|^2  \notag \\
%  &+ \alpha^2 \eta_L^2\E\| \nabla F(x_{t,k-1}^i;\xi_{b^r}) - \nabla \overline{F}_i(x_{t,k-1}^i) \|^2  \notag \\
%  &+ 6K(1-\alpha)^2\eta_L^2\E\|\nabla \tilde{F}_i(x_{t,k-1}^i) - \overline{F}_i(x_{t,k-1}^i)\|^2   \notag \\
%  &+ 6K\eta_L^2\E\|\nabla F_i(x_{t,k-1}^i)-\nabla F_i(x_t)\|^2 + 6K\eta_L^2\E\|\nabla F_i(x_t)\|^2  \notag \\
%  & \leq (1+\frac{1}{K-1})\E_t\|x_{t,k-1}^i-x_t\|^2+ \eta_L^2\sigma_{L}^2 +6K\eta_L^2(1-\alpha)^2\rho^2 \notag \\
%  &+6K\eta_L^2\sigma_G^2+6K\eta_L^2(A^2+1)\|\nabla f(x_t)\|^2 \, .
% \end{align} 
\begin{align}
&\E_t\|x_{t,k}^i-x_t\|^2   \notag\\
&\E_t\|x_{t,k-1}^i-x_t-\eta_L ( (1-\alpha) \nabla F(x_{t,k-1}^i;\xi_{b^u}) \notag \\
& + \alpha \nabla F(x_{t,k-1}^i;\xi_{b^r}) 
-(1-\alpha)\nabla \overline{F}_i(x_{t,k-1}^i)-\alpha \nabla \overline{F}_i(x_{t,k-1}^i) \notag \\
& +(1-\alpha)\nabla \tilde{F}_i(x_{t,k-1}^i)
    -(1-\alpha)\nabla \tilde{F}_i(x_{t,k-1}^i)
    +\nabla F_i(x_{t,k-1}^i) \notag \\
&    -\nabla F_i(x_t) + \nabla F_i(x_t)  ) \|^2  \notag \\
 & \leq (1+\frac{1}{2K-1})\E_t\|x_{t,k-1}^i-x_t\|^2 \notag \\
 &+ (1-\alpha)^2\eta_L^2\E\|\nabla F(x_{t,k-1}^i;\xi_{b^u}) 
 -\nabla \tilde{F}_i(x_{t,k-1}^i) \|^2  \notag \\
 &+ \alpha^2 \eta_L^2\E\| \nabla F(x_{t,k-1}^i;\xi_{b^r}) - \nabla \overline{F}_i(x_{t,k-1}^i) \|^2  \notag \\
 &+ 6K(1-\alpha)^2\eta_L^2\E\|\nabla \tilde{F}_i(x_{t,k-1}^i) - \overline{F}_i(x_{t,k-1}^i)\|^2   \notag \\
 &+ 6K\eta_L^2\E\|\nabla F_i(x_{t,k-1}^i)-\nabla F_i(x_t)\|^2 + 6K\eta_L^2\E\|\nabla F_i(x_t)\|^2  \notag \\
 & \leq (1+\frac{1}{K-1})\E_t\|x_{t,k-1}^i-x_t\|^2+ \eta_L^2\sigma_{L}^2 +6K\eta_L^2(1-\alpha)^2\rho^2 \notag \\
 &+6K\eta_L^2\sigma_G^2+6K\eta_L^2(A^2+1)\|\nabla f(x_t)\|^2 \, .
\end{align} 
% Unrolling the recursion, we finish the proof.
Unrolling the recursion, we obtain:
\begin{align}
&\E_t\|x_{t,k}^i-x_t\|^2 \leq \sum_{p=0}^{k-1}(1+\frac{1}{K-1})^p\left[\eta_L^2\sigma_{L}^2+6K\eta_L^2(1-\alpha)^2\rho^2 \right. \notag \\
&\left. +6K\eta_L^2\sigma_{G}^2+6K(A^2+1)\|\eta_L\nabla f(x_t)\|^2\right] \notag \\
&\leq (K-1)\left[(1+\frac{1}{K-1})^K-1\right]\left[\eta_L^2\sigma_{L}^2 \right.\notag \\
&\left. +6K\eta_L^2(1-\alpha)^2\rho^2+6K\eta_L^2\sigma_{G}^2+6K(A^2+1)\|\eta_L\nabla f(x_t)\|^2\right] \notag \\
&\leq 5K(\eta_L^2\sigma_L^2+6K\eta_L^2(1-\alpha)^2\rho^2+6K\eta_L^2\sigma_G^2) \notag \\
&+ 30K^2(A^2+1)\eta_L^2\|\nabla f(x_t)\|^2 \, .
\end{align}
\end{proof}

\section{Numerical results}
\label{sec experiments}

In this section, we present simulation results to validate the performance of the FedRec+ and compare it with the vanilla pseudo method (FedRec). We use two widely used benchmark datasets for recommendation including ML-100K and ML-1M. ML100K contains 100, 000 ratings of 1, 682 movies from 943 users; ML1M contains 1, 000, 209 ratings of 3, 952 movies from 6, 040 users, while both have rating levels of $[1,2,\cdots,5]$. We process each dataset follows the setting in ~\cite{liang2021fedrec++}:  (i) randomly dividing the dataset into five equal parts, (ii) using four parts for training and one part for testing, (iii) repeating this process four times to obtain five distinct sets of training and test data. Our experimental analysis is based on these five datasets, and we report the average performance across all five.

In our experiment, we use the PMF~\cite{mnih2007probabilistic} as the backbone model and we use three commonly used evaluation metrics, i.e., MAE, RMSE, and NMSE, for performance evaluation.

% Table~\ref{comparison table} shows that FedRec+ outperforms FedRec in all three recommendation metrics. This implies that our method effectively reduces the noise caused by pseudo items.
Table~\ref{comparison table} demonstrates the superior performance of FedRec+ over FedRec across all three recommendation metrics. This indicates our proposed algorithm is effective.

\begin{table}[!t]
\vspace{-.5em}
		\caption{Recommendation performance of FedRec and FedRec+ in ML1M and ML100K.}
		\label{comparison table}
  \vspace{-1.em}
  \begin{center}
  \resizebox{.46\textwidth}{!}{%
		\begin{tabular}{c|c|ccccccccc}
			% \toprule
			\toprule
			\multirow{1}{*}{Data}&Algorithm
			&MAE&RMSE &NMSE \cr
			\midrule
            \midrule
			\multirow{2}*{ML-1M}
			& FedRec &  0.8962\resizebox{!}{.5\normalbaselineskip}{{\transparent{0.5}±0.0002}}  &1.1178\resizebox{!}{.5\normalbaselineskip}{{\transparent{0.5}±0.0003}}  & 0.0893\resizebox{!}{.5\normalbaselineskip}{{\transparent{0.5}±0.0005}} \cr
			~ & FedRec+  &0.8929\resizebox{!}{.5\normalbaselineskip}{{\transparent{0.5}±0.0002}}  &1.1014\resizebox{!}{.5\normalbaselineskip}{{\transparent{0.5}±0.0003}}  & 0.0879\resizebox{!}{.5\normalbaselineskip}{{\transparent{0.5}±0.0009}} \cr
			\midrule
			\multirow{2}*{ML-100K}
			& FedRec &1.0120\resizebox{!}{.5\normalbaselineskip}{{\transparent{0.5}±0.0002}}   &1.3179\resizebox{!}{.5\normalbaselineskip}{{\transparent{0.5}±0.0021}}  & 0.1236\resizebox{!}{.5\normalbaselineskip}{{\transparent{0.5}±0.0009}} \cr
			~ & FedRec+  &0.9109\resizebox{!}{.5\normalbaselineskip}{{\transparent{0.5}±0.0021}}  &1.1713\resizebox{!}{.5\normalbaselineskip}{{\transparent{0.5}±0.0022}}  & 0.1006\resizebox{!}{.5\normalbaselineskip}{{\transparent{0.5}±0.0008}} \cr
			\bottomrule
			% \bottomrule
		\end{tabular}
  }
 \end{center}
\end{table}

% \begin{figure}
%     \centering
%     \includegraphics{Fig_illustration/FRS_pseudonumber_cut.pdf}
%     \caption{Caption}
%     \label{num of pseudo}
% \end{figure}

\begin{figure}
    \begin{center}
    \centerline{\includegraphics[width=0.8\columnwidth]{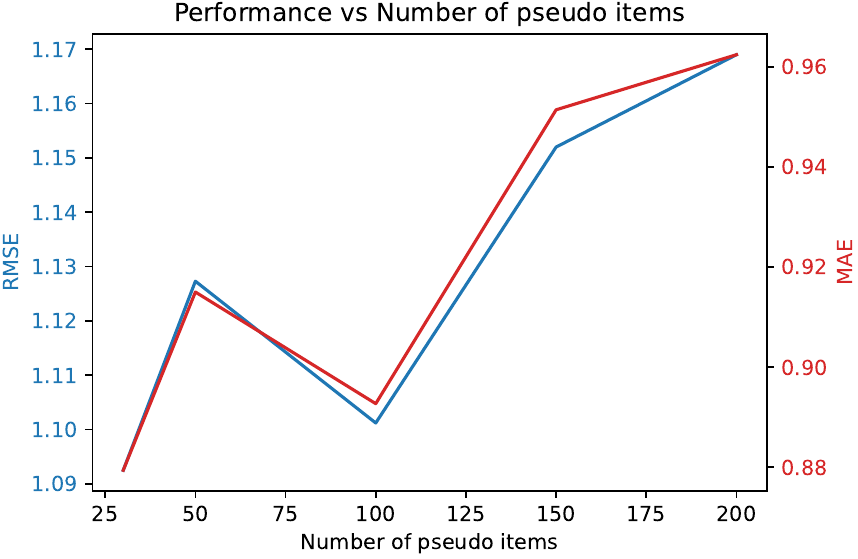}}
    \caption{Performance of FedRec+ under different numbers of pseudo items.}
    \label{num of pseudo}
    \end{center}
    \vspace{-1.em}
\end{figure}

Figure~\ref{num of pseudo} illustrates the impact of varying numbers of pseudo items on ML1M. It indicates that higher numbers of pseudo items result in increased noise and consequently worse performance, aligning with our theoretical analysis.

% When the EquitSet doesn't work, one possibility is the algorithm is only useful for some specific tasks as we only change the implementation in ML1M dataset. We may use the item-based local recommendation to generate similar items as pseudo-items, or seek help from the subset selection function.

% Function value (FV) oracle, targets at learning $F_{\theta}(S;V)$ to fit the utility explicitly(Equitset is an implicit way which may be the reason for bad performance), under the supervision of data in the form of ${(S_i, f_i)}$ for a fixed ground set $V$, where $f_i$ is the true utility function value of the subset $S_i$. 

% How about Greedy submodular selection?

% \paragraph{Local Rank.} Utilize the local rated items, to learn a personalized item-based similarity recommendation network. For pseudo items generation.

% \begin{itemize}
%     \item Describe the experimental setup and evaluation metrics used in the experiments
%     \item Present the experimental results for both the optimal set oracle and weight aggregation methods
%     \item Analyze the results and provide insights into the performance of the proposed techniques
% \end{itemize}

% New decision:
% We decide to use local item-based recommendation to replace the implicit way subset selection (do not know how to do the explicit way)

% new: greedy submodular selection for items.
% A better way: active learning selects items.

\section{Conclusion and future works}
\label{sec conclusion}

% In this paper, we study an emerging problem, i.e., privacy-aware federated recommendation with explicit feedback. In particular, we propose a privacy-preserving framework called FedRec+, for which we use the feature similarity to generate the low-noise pseudo items without any communication between clients. In addition, we use the Wasserstein Distance to model an optimization problem w.r.t the aggregation probability. The aggregation method helps tackle the heterogeneity of FRS. We also conduct convergence analysis to show the effect of pseudo items and aggregation probability. Our FedRec+ is a generic solution that is integratable with other privacy-aware recommendation methods such as differential privacy~\cite{dolui2019towards}. 
% Experimental results on public datasets show the effectiveness of our FedRec+.
This paper focuses on the problem of privacy-aware federated recommendation with explicit feedback. We propose FedRec+, a privacy-preserving framework that addresses this issue. FedRec+ utilizes feature similarity to generate low-noise pseudo items without client communication. Furthermore, we employ the Wasserstein Distance to optimize the aggregation probability, which helps handle the heterogeneity of the federated recommendation system. Convergence analysis is conducted to demonstrate the impact of pseudo items and aggregation probability. FedRec+ is a versatile solution that can be combined with other privacy-aware recommendation methods, such as differential privacy~\cite{dolui2019towards}. Our experimental results, based on public datasets, validate the effectiveness of FedRec+.

\section{ACKNOWLEDGEMENTS}
This work is supported in part by the National Natural Science Foundation of China under Grant No. 62001412, in part by the funding from Shenzhen Institute of Artificial Intelligence and Robotics for Society, in part by the Shenzhen Key Lab of Crowd Intelligence Empowered Low-Carbon Energy Network (Grant No. ZDSYS20220606100601002), and in part by the Guangdong Provincial Key Laboratory of Future Networks of Intelligence (Grant No. 2022B1212010001).

% \newpage
\bibliography{example_paper}
\bibliographystyle{plain}

%%%%%%%%%%%%%%%%%%%%%%%%%%%%%%%%%%%%%%%%%%%%%%%%%%%%%%%%%%%%%%%%%%%%%%%%%%%%%%%
%%%%%%%%%%%%%%%%%%%%%%%%%%%%%%%%%%%%%%%%%%%%%%%%%%%%%%%%%%%%%%%%%%%%%%%%%%%%%%%
% APPENDIX
%%%%%%%%%%%%%%%%%%%%%%%%%%%%%%%%%%%%%%%%%%%%%%%%%%%%%%%%%%%%%%%%%%%%%%%%%%%%%%%
%%%%%%%%%%%%%%%%%%%%%%%%%%%%%%%%%%%%%%%%%%%%%%%%%%%%%%%%%%%%%%%%%%%%%%%%%%%%%%%
% \clearpage
% \appendix
% \onecolumn
% \input{Appendix}

% You can have as much text here as you want. The main body must be at most $8$ pages long.
% For the final version, one more page can be added.
% If you want, you can use an appendix like this one, even using the one-column format.
%%%%%%%%%%%%%%%%%%%%%%%%%%%%%%%%%%%%%%%%%%%%%%%%%%%%%%%%%%%%%%%%%%%%%%%%%%%%%%%
%%%%%%%%%%%%%%%%%%%%%%%%%%%%%%%%%%%%%%%%%%%%%%%%%%%%%%%%%%%%%%%%%%%%%%%%%%%%%%%

\end{document}